\documentclass[twoside,11pt]{article}

\usepackage{jmlr2e}
\usepackage[margin=1in]{geometry}
\usepackage{amsmath}

\usepackage[english]{babel}
\usepackage[utf8]{inputenc}
\usepackage{graphicx}
\usepackage[colorinlistoftodos]{todonotes}
\usepackage{enumerate}
\usepackage[section]{placeins}
\usepackage{here}
\usepackage{verbatim}
\usepackage[font=small,labelfont=bf]{subcaption}
\usepackage[font=small,labelfont=bf]{caption}

\usepackage{titling}
\usepackage{blindtext}
\usepackage{hyperref}
\usepackage[title]{appendix}
\begin{document}
\title{Learning Directed Graphical Models from Gaussian Data }
\author{Katherine Fitch\thanks{Katherine Fitch is with the Chair of Operations Research, Technical University of Munich, Munich, Germany. katie.fitch@tum.de} }

    \maketitle
\begin{abstract}
In this paper, we introduce a  new directed  graphical model from Gaussian data: the Gaussian graphical interaction model (GGIM). The development of this  model comes from considering stationary Gaussian processes on graphs, and leveraging the equations between the resulting steady-state covariance matrix and the Laplacian matrix representing the interaction graph. Through the presentation of conceptually straightforward theory, we develop the new model and provide interpretations of the edges in the graphical model in terms of statistical measures. We show that when restricted to undirected graphs, the Laplacian matrix representing a GGIM is equivalent to the standard inverse covariance matrix that encodes conditional dependence relationships. Furthermore, our approach leads to a natural definition of \emph{directed} conditional independence of two elements in a stationary Gaussian process. We demonstrate that the problem of learning sparse GGIMs for a given observation set can be framed as a LASSO problem. By comparison with the problem of inverse covariance estimation, we prove a bound on the difference between the covariance matrix corresponding to a sparse GGIM and the covariance matrix corresponding to the $l_1$-norm penalized maximum log-likelihood estimate. Finally, we consider the problem of learning GGIMs associated with sparse directed conditional dependence relationships. In all, the new model  presents a novel perspective on directed relationships between variables and significantly expands on the state of the art in Gaussian graphical modeling.

\vspace{5mm}
\end{abstract}

\section{Introduction}
Classically, a Gaussian graphical model (GGM) of an observation set is a graph that encodes conditional independence relationships between variables, where each variable is a node in the graph \citep{whittaker2009}. That is, the presence of an edge between nodes $i$ and $j$ indicates that variables $i$ and $j$ are conditionally dependent given the remaining variables. Consequently, the absence of an edge indicates conditional independence between $i$ and $j$. It is well known that the sparsity pattern of the off-diagonal elements of the precision matrix (inverse covariance matrix) encodes these conditional independence relationships. 

Learning the structure of the conditional independence graph is accomplished by finding the precision matrix that balances sparsity, often with respect to the $l_1$-norm, with maximizing the Gaussian log-likelihood function associated with the set of i.i.d.\ observations \citep{Friedman2008, banerjee2008}. The non-zero off-diagonal entries of the resulting precision matrix are then taken to be edges in the underlying conditional independence graph. As the covariance matrix is by definition symmetric, its inverse is also symmetric and therefore this approach is limited to producing only undirected graphs. The precision matrix is in fact a Laplacian matrix for an undirected graph, and the conditional independence graph can also be viewed as an undirected graph of interactions between variables. 

The problem of estimating inverse covariance matrices has been extensively studied and applied to applications ranging from analysis of functional brain connectivity \citep{Kruschwitz2015} to speech recognition \citep{zhang2013} to computer vision applications \citep{Souly2016}. Among estimators that regularize the Gaussian log-likelihood by the $l_1$-norm, there are algorithmic approaches that employ block coordinate descent \citep{banerjee2008}, the graphical LASSO  \citep{Friedman2008}, quadratic approximation \citep{Hsieh2011}, second order Newton-like methods \citep{Oztoprak2012}, and many others. Additional methods of sparity promotion include greedy forward-backward search to determine the location of zeros in the precision matrix \citep{Lauritzen1996}, $l_0$-norm regularization \citep{Marjanovic2015}, and $l_q$ penalization \citep{Marjanovic2014}, where $0 \leq q < 1$.  

While GGMs and sparse inverse covariance estimation provide a straightforward and generalizable framework for approximating conditional independence relationships between variables from a sample set of observations, the question remains of how to graphically model directed interactions. For example, consider a random vector with three variables, $a$, $b$, and $c$, where $a$ influences $b$, $b$ influences $c$, and $c$ influences $a$. The conditional independence graph would contain undirected edges between all three nodes, however, the true sparsest model of interaction would be a directed cycle from $a$ to $b$ to $c$ to $a$.  This is to say that conditional dependence between any two variables  does not imply that the two variables influence each other equally, therefore,  inverse covariance estimation provides limited information on the interactions between variables. 

Learned directed graphical models such as Bayesian networks \citep{barber2012}, and linear Gaussian structure equation models (SEMs) \citep{Bollen1989} overcome some of the limitations of inverse covariance estimation by representing casual relationships between variables as edges in directed acyclic graphs (DAGs). While these models have been successfully and broadly implemented for a wide variety of applications, they still have non-trivial disadvantages. Most notably, due to their hierarchical nature, neither Bayesian networks nor SEMs can directly model cyclic interactions. As a result, the simple directed cycle from the previous paragraph cannot be accurately reproduced through these methods. This is an issue because the ability to learn and model cyclic interactions is critical to our understanding of complex systems. In a biological system, for example, cyclic interactions can be indicative of  regulatory feedback loops. Thus, when attempting to learn graphical structure from data, the failure to identify feedback loops results in an incomplete understanding of the fundamental mechanisms of the biological system. There is, therefore, a need for learned graphical models of Gaussian data that represent directed relationships between variables without topological restrictions.

\subsection{Contributions}
In this paper we significantly expand the modeling capabilities for Gaussian data by exploring features that can be modeled as directed relationships, namely interactions between variables and conditional independence, and developing a consistent theoretical framework uniting the model with traditional GGMs. This theoretical framework arises from considering a sample covariance matrix to be the steady-state covariance matrix of a stationary Gaussian process on a graph. We then examine the relationship between the inverse covariance matrix and the Laplacian matrix representing the interaction graph. The subtle perspective shift to stationary Gaussian processes allows us to leverage matrix equations to a great advantage, as we are able to unite the concepts of interactions between variables and conditional dependence without the need for applying complex rules to the nodes and edges in the graph. The connection between network topology of a stationary Gaussian process and inverse covariance matrices has been previously established (for example, \citep{Hassan2016}). However, thus far the focus has been limited to applying techniques from sparse inverse covariance estimation to the problem of \emph{undirected} network identification. 

The main contributions of this paper include the introduction of a new directed graphical model, the \emph{Gaussian graphical interaction model} (GGIM), and the definition of directed conditional independence. 
We demonstrate that the problem of learning $l_1$-norm sparse GGIMs for a given observation set can be framed as a LASSO problem and provide a bound on the difference between the covariance matrix estimate corresponding to a GGIM and the maximum likelihood estimate (MLE) of the covariance matrix, as well as the $l_1$-norm penalized maximum log-likelihood estimate  of the covariance matrix. 

The GGIM model is mathematically elegant and intuitive, and learning sparse models is simple to implement. As the GGIM model is not restricted to the class of directed acyclic graphs, it is able to model a large variety of general relationships between variables that are not possible with existing graphical models.  

\subsection{Example: Cell signaling network} \label{sec:ex}
To motivate the enhanced graphical modeling potential of GGIMs, we demonstrate the new model when applied to the same dataset as used by the authors of  \citep{Friedman2008} in their paper introducing the GLASSO algorithm for sparse precision matrix estimation. The dataset is from \citep{sachs2005} and contains flow cytometry data of $n=7466$ cells and $p=11$ proteins. Figure \ref{fig:gull} reproduces the classic signaling network  \citep[Figure~2]{sachs2005}. 

The conditional independence graph from \citep{Friedman2008} is shown in Figure \ref{fig:tiger}. To compare, we show the new graph model introduced in this paper, GGIM,  in Figures \ref{fig:mouse}. Details for computing the graph \ref{fig:mouse} can be found in Section \ref{sec:det}. All graphs contain 18 edges, where, in the directed setting, an edge from node $i$ to $j$ is counted independently from an edge $j$ to $i$.  Solid lines indicate edges that are present in the classic signaling graph and light gray dashed lines indicate edges that are not present in the classic signaling graph. In graph \ref{fig:mouse}, black lines indicate edges identified in the same orientation as the classic signaling graph. Dark gray lines indicate edges identified in the reverse orientation as the classic signaling graph.

The GGIM (\ref{fig:mouse}) graph captures more  of the edges of Figure (\ref{fig:gull}) than the undirected conditional independence graph (\ref{fig:tiger}), with the majority of correctly identified edges also in the correct orientation.
\begin{figure}[h!]
    \centering
    \begin{subfigure}[b]{0.25\textwidth}
        \includegraphics[width=\textwidth]{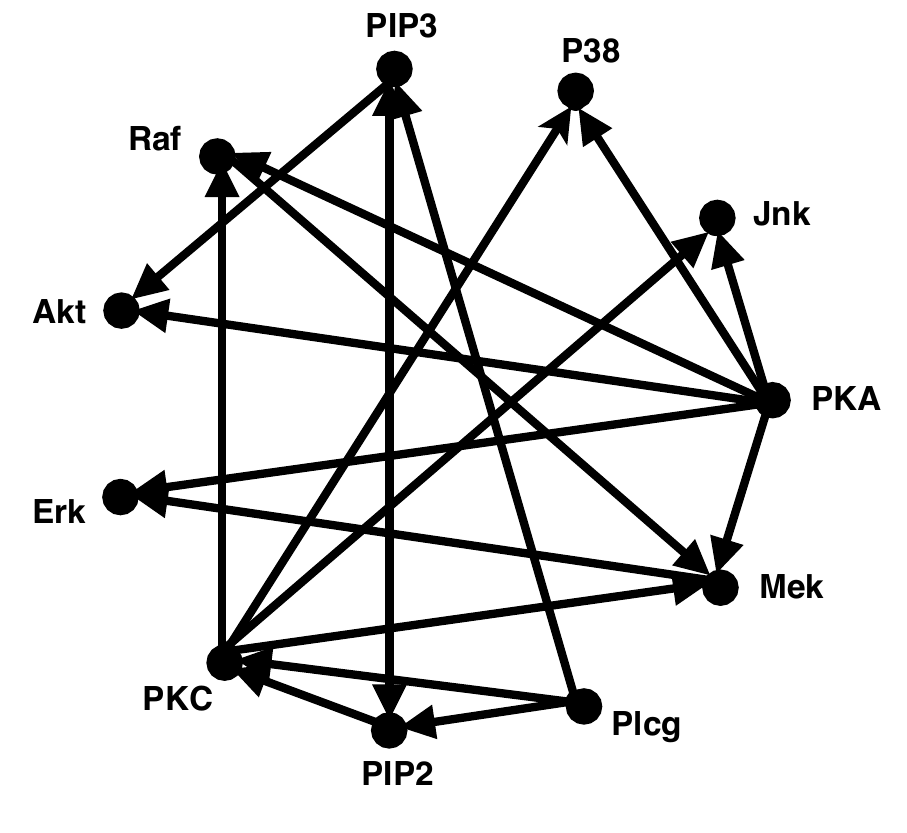}
        \caption{Cell signaling graph from \citep{sachs2005}}
        \label{fig:gull}
    \end{subfigure}
    ~ 
    \begin{subfigure}[b]{0.25\textwidth}
        \includegraphics[width=\textwidth]{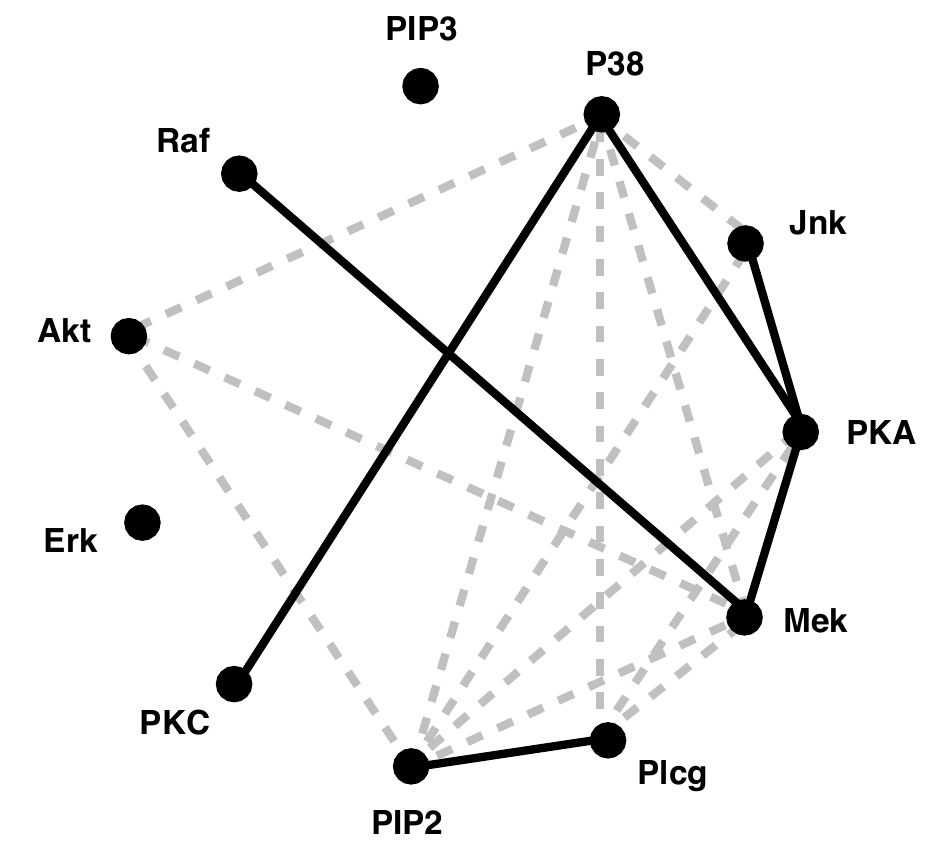}
        \caption{Undirected conditional independence graph from \citep{Friedman2008}}
        \label{fig:tiger}
    \end{subfigure}
    
    ~ 
    \begin{subfigure}[b]{0.25\textwidth}
        \includegraphics[width=\textwidth]{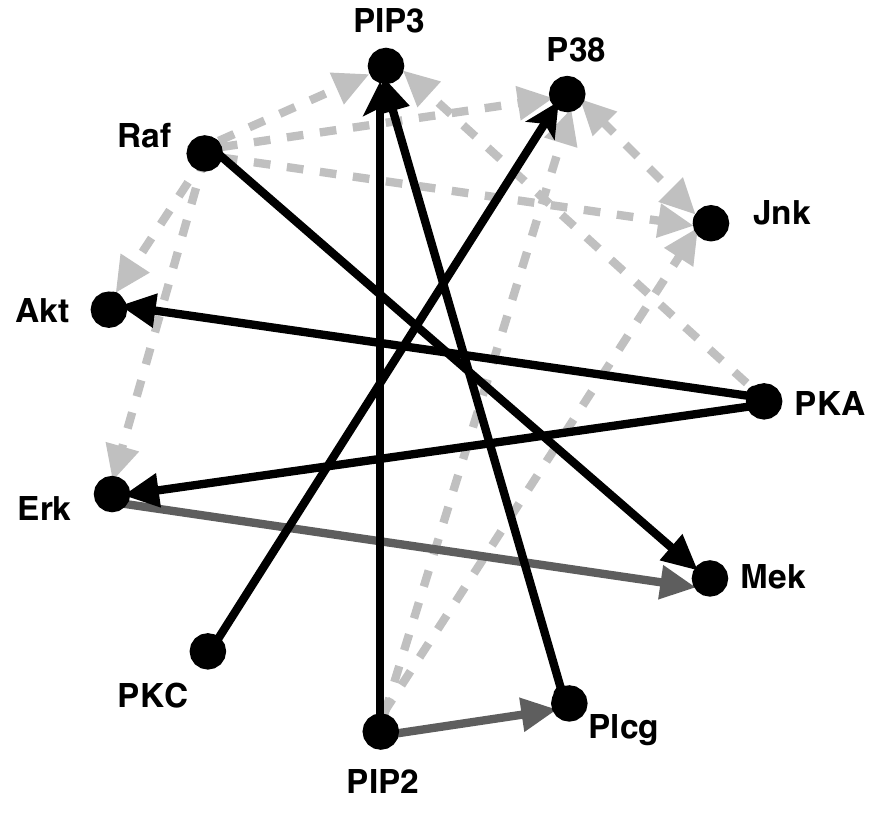}
        \caption{Directed GGIM}
        \label{fig:mouse}
    \end{subfigure}
    \caption{Comparison of graphical models from cell interaction data. (a). Classic signaling graph from \citep{sachs2005}. (b). Conditional independence graph generated by the GLASSO algorithm \citep{Friedman2008}. (c). Directed GGIM. Solid lines indicate edges that are present in the classic signaling graph. Dashed light gray lines indicate edges that are not present in the classic signaling graph.  In graph (c), black lines indicate edges identified in the same orientation as the classic signaling graph. Dark gray lines indicate edges identified in the reverse orientation as the classic signaling graph.}\label{fig:animals}
\end{figure}

\begin{remark}
\textup{
It is noted that GGIMs inherently differ from Bayesian networks and linear Gaussian SEMs. The models in this paper arise from the steady-state characteristics of a set of equations where the \emph{rate of change} of a variable is a linear combination of the state of the remaining variables. SEMs arise from a set of equations where the \emph{state} of a variable is a linear combination of the state of parent variables. Bayesian networks and SEMs model strictly hierarchical relationships through directed edges that represent causality, whereas GGIMs model the interactions between variables. The hierarchical nature of Bayesian networks and SEMs mandates that they are represented by DAGs, while no such restriction is placed on GGIMs, which can have any general graph structure. Due to these conceptual differences we do not seek to apply notions such as Markov equivalence \citep{chickering2002}, morality \citep{lauritzen1988}, or faithfulness \citep{Zhang2003} to GGIMs. 
} \end{remark}

\subsection{Outline}
The remainder of the paper is organized as follows. Background and notation are provided in Section \ref{sex:not}. In Section \ref{sec:covar} we develop relevant theory and introduce the Gaussian Graphical Interaction Model. In Section \ref{sec:what} we define directed conditional independence between elements in a stationary Gaussian process. Section \ref{sec:det} demonstrates how learning GGIMs can be framed as a LASSO problem and provides a bound on the estimated covariance matrix relative to the MLE of the covariance matrix for the GGIM. In Section sec:sparse we discuss learning GGIMs associated with sparse unerlying directed conditional dependence relationships. We conclude with final remarks in Section \ref{sec:final}.

\section{Background and notation}\label{sex:not}
Let $\mathcal{G} = (\mathcal{V}, \mathcal{E}, A)$ be a connected, directed graph representing interactions between variables, where $\mathcal{V} = \{1,2,..., p\}$ is the set of variables (also refered to as nodes), and $\mathcal{E} \subseteq \mathcal{V} \times \mathcal{V}$ is the set of $m$ edges.  $A \in \mathbb{R}^{p \times p} $ is the adjacency matrix where element $a_{i,j}$ is the weight on edge $(i,j)$. If $(i,j) \in \mathcal{E}$ then $a_{i,j} > 0 $; otherwise $a_{i,j} = 0$.  The \emph{out-degree} of node $i$ is calculated as $d_i = \sum_{j=1}^p a_{i,j}$.  The out-degree matrix is a diagonal matrix of node out-degrees, $D = \text{diag}\{d_1, d_2, ...,d_p\}$. The associated directed \emph{Laplacian} matrix is defined as $L=D-A$.  This  corresponds to a `sensing' interpretation of node interactions where an edge from $i$ to $j$ indicates that node $i$ senses the state of node $j$. We primarily consider this interpretation of edge orientation to maintain consistency with previous research on Gaussian processes on directed graphs, e.g. \citep{Fitch2019, young20161}. The transposed Laplacian corresponds to the `sending' interpretation where an edge from $i$ to $j$ indicates that the state of node $i$ is available to node $j$. 

Throughout the paper, partitioning of a given matrix, $G$, will follow
 \begin{align}
G =  \begin{bmatrix}
G_{a,a} &G_{a,b}\\
G_{a,b}^T & G_{b,b}
 \end{bmatrix}, \label{part}
 \end{align}
where $G_{a,a} \in \mathbb{R}^{2 \times 2}$. 

The sample mean of $n$ independent and identically distributed (i.i.d.)\ observations $X^{(1)}, \dots, X^{(n)}$ from $\mathcal{N}(\mu, \Sigma)$,  is 
\begin{align}
\overline{X} = \frac{1}{n} \sum_{i=1}^n X^{(i)}. \nonumber
\end{align}
The sample covariance matrix is 
\begin{align}
S = \frac{1}{n} \sum_{i=1}^n(X^{(i)} - \overline{X}) (X^{(i)} - \overline{X})^T, \nonumber
\end{align}
and is equal to the Maximum Likelihood Estimate (MLE) of the covariance matrix.

In general the Laplacian matrix, covariance matrix, and sample covariance matrix are all assumed to be full rank and corresponding to a weakly connected graph. The limiting case of GGIMs associated with a positive semi-definite covariance matrix is considered in Appendix \ref{app1}. In Section \ref{sec:det} we assume that the number of observations, $n$, is greater than the number of variables, $p$. However, we note that the corresponding LASSO problems can nevertheless be solved in the case $n \leq p$. 

The $l_1$- norm of a matrix is defined as the sum of absolute values of elements in the matrix and the $l_{\infty}$-norm of a matrix is defined as the maximum absolute value of elements in the matrix. Recall that a skew-symmetric matrix, $\kappa$, satisfies $\kappa^T = -\kappa$. We denote the set of all $\mathbb{R}^{p \times p}$ skew-symmetric matrices by $\text{so}(p)$.

\section{Gaussian graphical interaction models} \label{sec:covar}
We begin this section by reviewing the relationship between conditional independence and the inverse covariance matrix for a random vector $X$ with Gaussian distribution $\mathcal{N}(0,\Sigma)$. Then, we provide a lemma relating the inverse covariance matrix of $X$ to the undirected interaction graph for a stationary Gaussian process $\mathbf{x}$ with distribution $\mathcal{N}(0,\Sigma)$. Next, we discuss how the Gaussian process viewpoint allows us to consider directed interactions between variables. The section concludes with the definition for a new directed graphical model, the Gaussian graphical interaction model. 

There are multiple approaches for demonstrating that the sparsity pattern of the precision matrix is representative of the underlying undirected independence graph structure. For example, by using Schur complements it is possible to show that zeros in the precision matrix correspond to conditional independence relationships \citep{Anderson2003}. Consider a covariance matrix partitioned as in (\ref{part}),
 where $\Sigma_{a,a} \in \mathbb{R}^{2 \times 2}$ contains elements indexed by $1$, $2$, and let $P = \Sigma^{-1}$. Then, the submatrix $P_{a,a}$ can be written as 
 \begin{align}
 P_{a,a} = (\Sigma_{a,a} - \Sigma_{a,b}(\Sigma_{b,b})^{-1} \Sigma_{a,b}^T)^{-1} = (\Sigma_{a|b})^{-1}. \nonumber 
 \end{align}
  Therefore $P_{a,a}$ is the inverse of the conditional covariance, $\Sigma_{a|b}$, and two-by-two matrix inversion, $P_{1,2} = 0$ implies that  $\Sigma_{1,2|b} = 0$.
In other words, variables indexed by $i$ and $j$ are conditionally independent given the remaining variables if and only if  $P_{i,j} = \Sigma^{-1}_{i,j} = 0$.

We now demonstrate that the same relationship between the precision matrix and the underlying graph can be established from dynamic systems perspective. Rather than think of the $n$ i.i.d.\ samples as drawn from a multivariate normal distribution, they can be thought of as belonging to a stationary Gaussian process. Recall, for a stochastic process to be stationary its unconditional joint probability distribution is unchanged under time shifts. This further implies that the mean and covariance do not change over time.  Before stating the lemma regarding the relationship between inverse covariance matrix of a random vector and the interaction graphs of stationary Gaussian processes, we first provide a theorem from \citep{arnold1992}. In the following, $d\mathbf{W}$ represents increments drawn from independent Wiener processes.
\begin{theorem} (From \citep{arnold1992}) The solution of the equation
\begin{align}
d\mathbf{x} = (M(t) + a(t))\mathbf{x} dt + B(t)d\mathbf{W}, \;\;\mathbf{x}_{t_0} = c, \label{dyn1}
\end{align}
is a stationary Gaussian process if $M(t) \equiv M$, $a(t) \equiv 0$, $B(t) \equiv B$, the eigenvalues of $M$ have negative real parts, and $c$ is $\mathcal{R}(0,K)$-distributed, where $\Sigma$, the steady-state covariance matrix, is 
\begin{align}
\Sigma = \int_0^{\infty}e^{Mt}BB^Te^{M^Tt}dt, \nonumber
\end{align}
or equivalently the solution of the Lyapunov equation 
\begin{align}M \Sigma +\Sigma M^T = -BB^T. \nonumber
\end{align} 
\end{theorem}

\begin{lemma}
Consider a stationary Gaussian process $\mathbf{x}  \in \mathbb{R}^p$ with multivariate Gaussian distribution $\mathcal{N}(0,\Sigma)$ and a random vector $X \in \mathbb{R}^p$ with multivariate Gaussian distribution $\mathcal{N}(0,\Sigma)$. Then, up to a scaling, the unique symmetric Laplacian matrix corresponding to the steady-state undirected interaction graph associated with $\mathbf{x}$ is equal to the inverse covariance matrix associated with the random vector $X$.
\end{lemma}
\begin{proof}

Consider equation (\ref{dyn1}) in the context of a stochastic process on a graph by replacing $M$ with the negative graph Laplacian $-L$ and taking $B= \sigma I_p$.  The resulting equation, 
\begin{align}
d\mathbf{x} = -L\mathbf{x} dt + \sigma d \mathbf{W}, \label{dyn2}
\end{align} represents a noisy diffusion process where each node $i$ in the graph averages its state, $\mathbf{x}_i$, with the state of its neighbors, subject to noise in the state of each node with standard deviation $\sigma$. Without loss of generality, let $\sigma^2 = 2$.  As the eigenvalues of $-L$ have negative real parts, the solution $\mathbf{x}$ is a stationary process. The corresponding steady-state covariance matrix is the unique solution to
\begin{align}
L \Sigma + \Sigma L^T = 2 I_p. \label{lyap1}
\end{align}
When restricting to only undirected graphs ($L$  symmetric), the unique solution $\Sigma =  L^{-1}$ satisfies (\ref{lyap1}). \end{proof}

Therefore, we have demonstrated that for undirected graphical models, the stationary Gaussian diffusion process represented by the steady-state dynamics of $\mathbf{x}$ recovers the relationship between the precision matrix and graph structure, up to a scaling. More specifically, by the relationship $L = \Sigma^{-1}$, the non-zero elements of the inverse covariance matrix indicate both conditional dependencies and edges in the underlying interaction graph.

Returning to the stationary Gaussian process (\ref{dyn2}), we can ask what happens when $L$ represents a directed graph and is no longer symmetric. In this case, the solution $\Sigma$ to (\ref{lyap1}) still exists, is unique, and is symmetric,  but is no longer a scale of the inverse Laplacian matrix. In the reverse direction, given a covariance matrix, $\Sigma$, a corresponding $L$ is not unique. We formalize the relationship between a $L$ and $\Sigma$ in the following proposition, which follows from considering the relationships provided for rank deficient matrices in \citep{Fitch2019} in the context of full-rank $L$ and $\Sigma$.
\begin{proposition} \label{prop1}
Consider a stationary Gaussian process $\mathbf{x}  \in \mathbb{R}^p$ with multivariate Gaussian distribution $\mathcal{N}(0,\Sigma)$. Then the Laplacian matrices corresponding to the family of steady-state equivalent interaction graphs between the $p$ variables of $\mathbf{x}$ are given as a function of $\kappa \in \mathbb{R}^{p \times p}$ skew-symmetric matrices, by
\begin{align}
L(\kappa) = (I_p + \kappa)\Sigma^{-1}. \label{lk}
\end{align}
\end{proposition}
\begin{proof}
 Direct substitution yields 
\begin{align}
L \Sigma + \Sigma L^T = (I_p + \kappa)\Sigma^{-1}\Sigma + \Sigma \Sigma^{-1} (I_p - \kappa) = I_p + \kappa +I_p - \kappa = 2 I_p. \nonumber
\end{align} It can be shown (i.e. \citep{Barnett1967}) that there do not exist any Laplacian matrices corresponding steady-state equivalent interaction graphs outside of those that can be decomposed as (\ref{lk}). \end{proof}
This leads to the definition of the Gaussian Graphical Interaction Model:
\begin{definition}[Gaussian Graphical Interaction Model (GGIM)] \label{def1} Consider a stationary Gaussian process with multivariate distribution $\mathcal{N}(0,\Sigma)$. Then a Gaussian Graphical Interaction Model for the stationary process is the graph induced by a Laplacian matrix $\hat{L}$ where
\begin{align} 
&\hat{L} = (I_p + \kappa)\Sigma^{-1}, \;\;\kappa \in \text{so}(p).
\end{align}
\end{definition}

\begin{remark} 
\textup{Similarly to most structure learning problems, learning a GGIM is inherently indeterminate, as there are an infinite number of matrices $\hat{L}$ that can satisfy the definition of a GGIM for a given stationary Gaussian process. In practice, we will add a regularization term penalizing the $l_1$-norm of $\hat{L}$. While constructing a directed interaction model based on a symmetric covariance matrix may seem at first counter intuitive, we note that for the $l_1$-norm, finding the sparsest directed matrix $L$ that satisfies (\ref{lyap1}) inherently provides a generalization of a very natural assumption on edge directions. To illustrate, consider a two variable example consistent with the sending interpretation of variable interaction (corresponding to $L^T$). If $\Sigma_{1,1} > \Sigma_{2,2}$ then the resulting $l_1$-norm sparsest GGIM will contain an edge from $2$ to $1$. In a sense, this assumes the variable with less uncertainty passes information to the variable with higher uncertainty. The directed edge in this example represents the sparsest interaction that could lead to the covariance matrix $\Sigma$.} \end{remark}

 \section{Directed conditional dependence in graphical models}\label{sec:what}
  In this section, we show that the submatrix equation from (\ref{lyap1}) pertaining to the conditional covariance matrix of two nodes allows us to write the off-diagonal elements of $P$ in terms of two directed components. This then naturally leads to a notion of directed conditional dependence. 

We begin by rewriting equation (\ref{lyap1}) as 
\begin{align}
\Sigma^{-1} L + L^T \Sigma^{-1} = 2(\Sigma^{-1})^2. \label{lyapi}
\end{align}
Let $\Sigma$ and $L$ be partitioned according to (\ref{part}), where $\Sigma_{a,a}, \; L_{a,a} \in \mathbb{R}^{2 \times 2}$ are indexed by $1,2$, and express $\Sigma^{-1}$ as 
 \begin{align}
 \Sigma^{-1} = P = \begin{bmatrix}
(\Sigma_{a|b})^{-1} & - (\Sigma_{a|b})^{-1}\Sigma_{a,b} (\Sigma_{b,b})^{-1} \\
-(\Sigma_{b,b})^{-1}\Sigma_{a,b}(\Sigma_{a|b})^{-1} &  (\Sigma_{b,b})^{-1}+(\Sigma_{b,b})^{-1}\Sigma_{a,b}(\Sigma_{a|b})^{-1}\Sigma_{a,b} (\Sigma_{b,b})^{-1} 
\end{bmatrix}. \nonumber
\end{align}
Then, the top left submatrix of (\ref{lyapi}) can be expanded as
\begin{align}
&(\Sigma_{a|b})^{-1} L_{a,a} -(\Sigma_{a|b})^{-1}\Sigma_{a,b} (\Sigma_{b,b})^{-1}L_{b,a} + L_{a,a}^T(\Sigma_{a|b})^{-1} - L_{a,b}^T (\Sigma_{b,b})^{-1}\Sigma_{a,b}(\Sigma_{a|b})^{-1} \nonumber \\&= 2(\Sigma_{a|b})^{-2} + 2( \Sigma_{a|b})^{-1}\Sigma_{a,b} (\Sigma_{b,b})^{-1}(\Sigma_{b,b})^{-1}\Sigma_{b,a}(\Sigma_{a|b})^{-1}. \nonumber
\end{align}
Multiplying both sides by $\Sigma_{a|b}$ yields
\begin{align}
&(L_{a,a} -\Sigma_{a,b} (\Sigma_{b,b})^{-1}L_{b,a})\Sigma_{a|b} + \Sigma_{a|b}(L_{a,a}^T - L_{a,b}^T (\Sigma_{b,b})^{-1}\Sigma_{a,b}) = 2I +2\Sigma_{a,b} (\Sigma_{b,b})^{-1}(\Sigma_{b,b})^{-1}\Sigma_{a,b}. \nonumber
\end{align}
Substituting $P_{a,b} =  - (\Sigma_{a|b})^{-1}\Sigma_{a,b} (\Sigma_{b,b})^{-1}$ and rearranging gives
\begin{align}
&\Big(L_{a,a} + \Sigma_{a,b}(\Sigma_{b,b})^{-1}(  P_{b,a} - L_{b,a})\Big) \Sigma_{a|b} 
+ \Sigma_{a|b}\Big(L_{a,a} + \Sigma_{a,b}(\Sigma_{b,b})^{-1}( P_{b,a} - L_{b,a})\Big)^T = 2I. \label{lyap3}
\end{align}
Let $\Phi_{a,a} = L_{a,a} + \Sigma_{a,b}(\Sigma_{b,b})^{-1}(  P_{b,a} - L_{b,a}) $. Then equation (\ref{lyap3}) can be written as
\begin{align}
    \Phi_{a,a} \Sigma_{a|b} 
+ \Sigma_{a|b}\Phi_{a,a} ^T = 2I. \label{lyap4}
\end{align}

From \citet{Barnett1967}, and analogously to (\ref{lk}),  $\Phi_{a,a}$ can also be written  as a function of a skew-symmetric matrix such that $\Sigma_{a|b}$ is a solution to the Lyapunov equation (\ref{lyap4})
\begin{align}
\Phi_{a,a} = L_{a,a} + \Sigma_{a,b}(\Sigma_{b,b})^{-1}(  P_{b,a} - L_{b,a}) = (I_2 + \Gamma)\Sigma_{a|b}^{-1} =   (I_2 + \Gamma)P_{a,a}, \label{gamma1}
\end{align}
where $\Gamma \in \mathbb{R}^{2 \times 2}$ is a skew symmetric matrix.  $\Gamma$ has only one free variable, 
\begin{align}
\Gamma_{1,2} = - \Gamma_{2,1} =  \frac{\Phi_{k,l} - \Phi_{l,k}}{P_{k,k} + P_{l,l}}. \label{gamma2}
\end{align}
Furthermore, combining (\ref{gamma1}) and (\ref{gamma2}) we can explicitly solve for the relationship between off-diagonal elements of the precision matrix $P_{k,l}$ with $\Phi_{k,l}$ and $\Phi_{l,k}$. That is,
\begin{align}
P_{k,l} = \Phi_{k,l}\frac{P_{k,k}}{P_{k,k}+P_{l,l}} +\Phi_{l,k}\frac{P_{l,l}}{P_{k,k}+P_{l,l}}. \label{prec_to_directed}
\end{align}
If two nodes, $k$, $l$ are conditionally independent then $P_{k,l} = 0$. Since $P_{k,k}, P_{l,l} > 0 $ (by positive definiteness of $\Sigma$), we have from equation (\ref{prec_to_directed}) that $k$ and $l$ will be conditionally independent if and only if  $\frac{\Phi_{k,l}}{P_{l,l}} = -\frac{\Phi_{l,k}}{P_{k,k}}$. Equation  (\ref{prec_to_directed}) demonstrates that we can essentially split undirected edges in an inverse covariance model into two directed components. This result is surprisingly elegant in its implication that there is a separation of conditional dependence into two directed parts. This relationship between $\Phi_{k,l}$, $\Phi_{l,k}$, and $P_{k,l}$ motivates the definition of directed conditional independence. 

\begin{definition}\emph{\textbf{Directed conditional independence of two elements in a stationary Gaussian process}} \\
Consider a stationary Gaussian process $\mathbf{x} \in \mathbb{R}^p$  with multivariate distribution $\mathcal{N}(0,\Sigma)$ over the the graph $\mathcal{G} = (\mathcal{V}, \mathcal{E}, A)$ with associated Laplacian matrix $L$ and precision matrix $P = \Sigma^{-1}$.  Additionally, consider two variables $k$, $l$, with $a=\{k,l\}$ and the set $b = \{\mathcal{V} \backslash \{k,l\} \}$. Let $\Phi_{a,a} = L_{a,a} + \Sigma_{a,b}(\Sigma_{b,b})^{-1}(  P_{b,a} - L_{b,a}) $. Then, $k$ and $l$ are \emph{conditionally independent in the direction $k$ to $l$} when $\Phi_{l,k} = 0$. Similarly, $k$ and $l$ are \emph{conditionally independent in the direction $l$ to $k$} when $\Phi_{k,l} = 0$.
\end{definition}

 To provide more intuition for the definition of directed conditional independence, we can substitute (\ref{prec_to_directed}) into the precision sub-matrix corresponding to the conditional distribution, yielding
\begin{align}
P_{a,a} & = \begin{bmatrix}
P_{k,k} & \frac{\Phi_{k,l}P_{k,k} +\Phi_{l,k}P_{l,l}}{P_{k,k}+P_{l,l}} \\ \frac{\Phi_{k,l}P_{k,k} +\Phi_{l,k}P_{l,l}}{P_{k,k}+P_{l,l}} & P_{l,l} 
\end{bmatrix}. \nonumber 
\end{align}
Let $0 < \{ \alpha_{k}, \; \alpha_l \} < 1$. 
 Then $P_{a,a} $ can be decomposed as 
\begin{align}
P_{a,a}  &=  P_{a,a}^{(l \rightarrow k)} + P_{a,a}^{(k \rightarrow l)} = \frac{P_{k,k}}{P_{k,k}+P_{l,l}} \begin{bmatrix}
 P_{k,k} & \Phi_{k,l} \\ \Phi_{k,l} &  P_{l,l} \end{bmatrix}
+ \frac{P_{l,l}}{P_{k,k}+P_{l,l}}  \begin{bmatrix}
 P_{k,k}  & \Phi_{l,k} \\ 
 \Phi_{l,k}  &  P_{l,l}
\end{bmatrix}. \label{psplit}
\end{align}
We call the decomposition \emph{valid} if $ P_{a,a}^{(l \rightarrow k)}$ and $P_{a,a}^{(k \rightarrow l)}$ are positive definite. 
A valid decomposition of $P_{a,a}$ permits $\Sigma_{a|b}$ to be written as the covariance matrix associated with the product of two Gaussian distributions: $\mathcal{N}(0,\Sigma^{(l \rightarrow k)}_{a | b})$, and $\mathcal{N}(0,\Sigma^{(k \rightarrow l)}_{a | b})$. That is
\begin{align}
\Sigma_{a|b} = \Big({\Sigma^{(l \rightarrow k)}_{a | b}} ^{-1} + {\Sigma^{(k \rightarrow l)}_{a | b}} ^{-1} \Big)^{-1}, \label{cond1}
\end{align}
where $\Sigma^{(k \rightarrow l)}_{a | b}$ represents the component of the conditional covariance $\Sigma_{a|b}$ arising from the conditional dependence of $l$ on $k$.

Equation (\ref{cond1}) is equivalent to (\ref{psplit}), where $ P_{a,a}^{(k \rightarrow l)} = {\Sigma^{(k \rightarrow l)}_{a | b}} ^{-1}$ and $ P_{a,a}^{(l \rightarrow k)} = {\Sigma^{(l \rightarrow k)}_{a | b}} ^{-1}$ are both symmetric matrices. By two-by-two matrix inversion we have that $\Sigma^{(k \rightarrow l)}_{k,l | b} = 0$ if and only if $ \Phi_{l,k} = 0$.  Similarly in the opposite direction $\Sigma^{(l \rightarrow k)}_{k,l | b} = 0$ if and only if $\Phi_{k,l} = 0$. 

\subsection{Directed conditional independence and interaction graphs}\label{sec:ciandig}
With the exceptions of the undirected case $L = P$, we see that $\Phi_{k,l} = 0$ does not in general imply that $L_{k,l} = 0$, and vice versa. Consequently, it is possible to have an edge $(k,l)$ in the interaction graph despite conditional independence in direction $k \rightarrow l$. Similarly, it is possible to have conditional dependence in direction $k \rightarrow l$, but no edge $(k,l)$ in the associated GGIM.

To illustrate, consider Figure \ref{three_node}, which illustrates three simple sparsity structures for three node GGIMs (column label $L$) with the corresponding conditional independence graphs (column $P$) and directed conditional independence graphs (column $\Phi$). All diagonal elements in $\Phi$ are assumed to be non-positive ($\Phi_{k,l} \leq 0$), which is consistent with non-negative conditional dependencies. Conditions on $L$ for increased sparsity in $\Phi$ are listed in the fourth column. Figure \ref{three_node} omits sparsity structures for $P$ and $\Phi$ that are equivalent under a relabeling of nodes. 
\begin{figure}[h!]
\begin{center}
\centerline{\includegraphics[width=.66 \columnwidth]{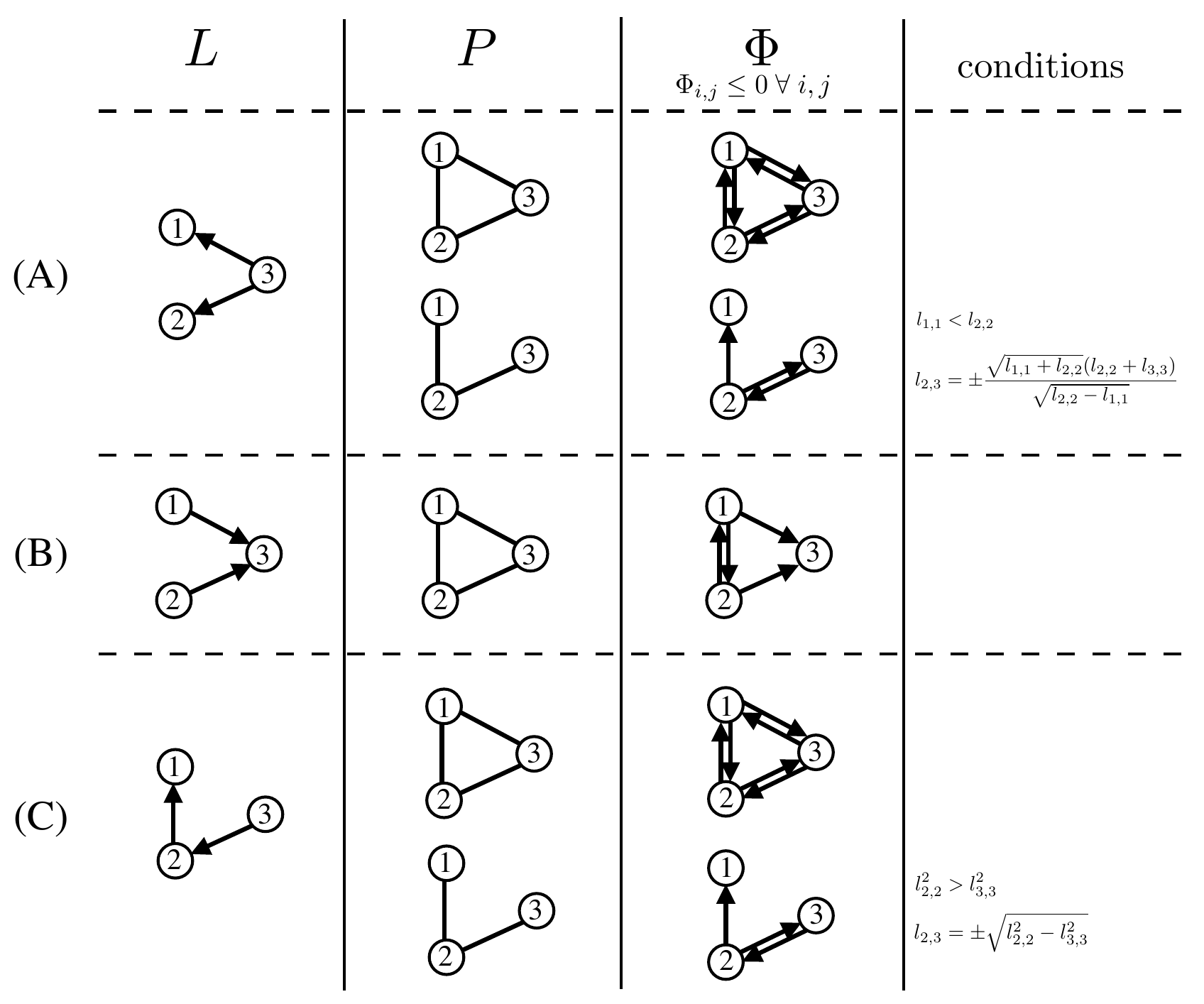}}
\caption{Examples of three node directed interaction graphs and corresponding conditional independence graphs and directed conditional independence graphs.}
\label{three_node}
\end{center}
\end{figure}

Despite being simple examples, the graphs in Figure \ref{three_node} provide a number of insights. The first, and perhaps least intuitive, is that the sparsity structure (A) has no set of edge weightings that yield conditional independence between the two child nodes. This is an example of how a Gaussian process induces coupling over time that leads to conditional dependence, even when there is no instantaneous connection between two nodes. Sparsity structure (B) demonstrates an interaction graph with a collider structure. As expected, the parents of node 3 are conditionally dependent. What we also see is that there are no edge weights on $L$ for which the parent nodes 1 or 2 will be directed conditionally dependent on the child node. 

Unlike Bayesian models, the relationships between $L$ and $\Phi$ in Figure \ref{three_node} can not be applied locally to directed triplets in larger graphs. Though inconvenient, this instinctively holds with our understanding of complex dynamical systems. In general, the relationship between two elements influences the entire system, rather than only their local neighborhood. 

\begin{remark} \textup{We note that the graphical models of directed conditional dependence considered in this paper are not identical to a composition of all two node projection graphs (or marginalization models) in \citet{Varando2020}. For example, in Figure \ref{three_node} (B), the composition of pairwise projection graphs is identical to the original sparsity structure and does not induce edges between 1 and 2. The projection graphs in \citet{Varando2020}, can be informally described in the following way. Consider a Gaussian process on a graph with associated steady-state covariance matrix $\Sigma$. Then, for a subset of nodes $c$, find the graph over the nodes in $c$ for which the steady-state covariance matrix is $\Sigma_{c,c}$. Therefore, when considering the limiting case of two node projection graphs, those in \citet{Varando2020} can be thought of as graphs of mutual dependence, providing a complementary picture to the conditional dependence graphs discussed here. }
\end{remark}


\section{Learning sparse GGIMs} \label{sec:det}

Typically the true covariance matrix, $\Sigma$ of a Gaussian process is unknown, and therefore the graphical models must be estimated from a set of observations. Informally, the optimization problems associated with learning GGIMs are stated as: Given observations from a stationary Gaussian process, jointly estimate $\hat{\Sigma}$ and $\hat{L}$ ($\hat{P}$) that yield the sparsest directed GGIM while sufficiently fitting the sample data. In the following section we formulate the GGIM learning problems mathematically and provide a bound on the covariance matrix estimate in the case of GGIM learning. 

To emphasize the theoretical continuity between GGIMs, and inverse covariance matrices, we adopt the $l_1$-norm sparsity objective that is often included in inverse covariance estimation. Naturally, there is no guarantee that the $l_1$-norm is truly the best penalization for all real world data sets. One could consider learning GGIMs with other sparsity promoting terms, however, this is outside the scope of this work.

\subsection{Sparse GGIM learning} \label{sec:ggiml}
Consider $n$ observations $\mathbf{x}^{(1)}, \dots, \mathbf{x}^{(n)}$ with sample covariance $S$ from a stationary Gaussian process with multivariate distribution $\mathcal{N}(0,\Sigma)$ . A GGIM estimate without any constraints on sparsity is a graph with Laplacian matrix $L$ for which (\ref{lyap1}) is satisfied with the maximum likelihood estimate (MLE) of the covariance matrix, $S$. That is, 
\begin{align}
LS + S L^T = 2I. \label{lyap7}
\end{align}
Equation (\ref{lyap7}) can be written in the form 
$
H \mathbf{z} = \mathbf{f}, 
$
 where $\mathbf{z} = \text{vec}(L)$, $\mathbf{f} = \text{vec}(2I)$ and $H \in \mathbb{R}^{1/2(p^2 +p)\times p^2}$. 
 
A sparse GGIM estimate is a graph induced by the the solution vector $\hat{\mathbf{z}}$ which minimizes a norm on the error term $\mathbf{f} - H \hat{\mathbf{z}}$, subject to sparsity penalization. Here, we consider the squared $l_2$-norm on the error and a sparsity penalization by way of the $l_1$-norm on $\hat{\mathbf{z}}$. This leads to the LASSO problem
 \begin{align}
 \hat{\mathbf{z}} = \min_{\mathbf{z}} \| \mathbf{f} - H \mathbf{z} \|_2^2 + \rho \|\mathbf{z} \|_1,\label{obj2}
 \end{align}
where $\rho \geq 0 $. The problem (\ref{obj2}) can be solved with standard algorithms such as coordinate descent \citep{Wu2008}. We omit a discussion on LASSO problems here and refer instead to \citep{Hastie2015} for an in-depth review of LASSO solution approaches, interpretations, and generalizations. 

The associated Laplacian matrix estimate $\hat{L}$ is the reshaped vector $\mathbf{\hat{z}}$ such that $\hat{L} \in \mathbb{R}^{p \times p}$. Then, the sparse GGIM estimate is the graph induced by Laplacian matrix $\hat{L}$. Note that we have suppressed dependence on $\kappa$ relative to Definition \ref{def1} because the problem of estimating $\hat{\kappa}$ and $\hat{\Sigma}$ such that $\hat{L}$ is sparse and (\ref{lyap7}) approximately holds is a more difficult problem than estimating  $\hat{L}$ directly. However, for an estimate $\hat{L}$ obtained by the solution to (\ref{obj2}), there exists a unique $\hat{\Sigma}$ and $\hat{\kappa}$ that can be calculated by solving $\hat{L} \hat{\Sigma} + \hat{\Sigma} \hat{L}^T = 2I$ and $\hat{\kappa} = \hat{L} \hat{\Sigma} - I. $ $\hat{L}$ as defined above is consistent with the sensing interpretation of node interaction, the more conventional sending interpretation (i.e. as applied in the example in Section \ref{sec:ex}) is obtained by $\hat{L}^T$.

\subsection{Sparse GGIM learning with bounded covariance}

In \citep{Friedman2008, banerjee2008}, the authors demonstrate that covariance matrix estimate that maximizes the Gaussian log-likelihood function subject to a $l_1$ penalization on the matrix inverse is a bounded additive perturbation to the unconstrained Gaussian MLE of the covariance, $S$. Here, we show that in the directed GGIM learning variant, the problem can be reformulated and solved such that the additive difference between $S$ and the estimated covariance matrix, $\hat{\Sigma}$, corresponding to a GGIM is bounded relative to the error on (\ref{lyap7}). This, in turn provides us with a bound on the difference between the penalized maximum likelihood estimates of \citep{Friedman2008, banerjee2008} and the covariance estimate corresponding to a sparse GGIM.

From (\ref{lyap1}) we have that in the unconstrained case without any penalty on the sparsity of $L$, the diagonal elements of $L$ satisfy
\begin{align}
L_{i,i} = \frac{1}{S_{i,i}}\Big(1 - \sum_{ j\neq i} L_{i,j}S_{i,j}\Big). \label{diagL}
\end{align}
Substituting (\ref{diagL}) into the equations for off-diagonal elements gives for every pair $j \neq k$,
\begin{align}
&\sum_{i \neq j} L_{j,i} \Big( \frac{S_{j,k} S_{j,i}}{S_{j,j}} - S_{k,i}\Big) + \sum_{l \neq k} L_{k,l} \Big( \frac{S_{j,k} S_{k,l}}{S_{k,k}} - S_{j,l}\Big) =  \frac{S_{j,k}}{S_{j,j}} +\frac{S_{j,k}}{S_{k,k}}.\label{const1}
\end{align}
Let $\mathbf{\zeta} \in \mathbb{R}^{(p^2-p )\times 1 }$ be the vector of off-diagonal elements of $L$. Let $\tilde{H} \in \mathbb{R}^{(p^2-p)/2 \times (p^2-p)} $ be the matrix where the elements in each row represent corresponding of $\mathbf{\zeta}$ for one constraint equation (\ref{const1}), and let $\mathbf{\beta} \in \mathbb{R}^{(p^2-p)/2 \times 1}$ be the vector of elements on the right side of each constraint equation (\ref{const1}). Then, the set of equations can be written as $\tilde{H} \mathbf{\zeta} = \mathbf{\beta}$ and the corresponding LASSO problem for the off-diagonal elements of $L$ is
 \begin{align}
 \hat{\mathbf{\zeta}} = \min_{\mathbf{\zeta}} \| \mathbf{\beta} - \tilde{H} \mathbf{\zeta} \|_2^2 + \rho \|\mathbf{\zeta} \|_1.\label{obj3}
 \end{align}
 Let $\hat{L}$ be such that the off-diagonal elements are the appropriate components of the reshaped vector  $\hat{\zeta}$ and the diagonal elements satisfy
 \begin{align}
\hat{L}_{i,i} = \frac{1}{S_{i,i}}\Big(1 - \sum_{ l\neq i} \hat{L}_{i,l}S_{i,l}\Big) + \epsilon_i, \label{diagL2}
\end{align}
where $\epsilon_i$ is chosen such that $L^{(k)} = (I \otimes \hat{L} + \hat{L}^T \otimes I)$ is strictly diagonally dominant. Note that it if $\nu_r = \max_j \sum_{i \neq j} |\hat{L}_{j,i}|$ and $\nu_c = \max_j \sum_{i \neq j} |\hat{L}_{i,j}|$, it suffices to set $\epsilon_i$ such that $L_{i,i} > \nu_r + \nu_c, \; \forall i$.

Before stating the theorem bounding the covariance matrix estimate $\hat{\Sigma}$ corresponding to the GGIM estimate $\hat{L}$, we first state a Lemma from \citep{varga1976}.
 \begin{lemma} \label{lemmasdd}
 (from \citep{varga1976}) 
Assume that $A \in \mathbb{C}^{p \times  p} $ is strictly diagonally dominant and let
$
\alpha = \min_i  \{|A_{i,i}| - \sum_{i \neq j } |A_{i,j}| \}.$
Then 
$
\|A^{-1} \|_{\infty} \leq \frac{1}{\alpha}.
$
 \end{lemma}
 \begin{theorem} \label{thmbound}
Let $\xi = \| \mathbf{f} - H \hat{\mathbf{\zeta}}' \|_2$ where $\hat{\zeta}' = \text{vec}(\hat{L})$. Furthermore, let $\hat{L}$ be such that $L^{(k)}$ is strictly diagonally dominant with $\alpha = \min_i \{|L^{(k)}_{i,i}| - \sum_{i \neq j } |L^{(k)}_{i,j}| \}$. Then, the covariance matrix implicitly estimated by (\ref{obj3}) and (\ref{diagL2}), $\hat{\Sigma}$, satisfying $\hat{L} \hat{\Sigma} + \hat{\Sigma} \hat{L}^T = 2 I_p$ is bounded with respect to the unconstrained MLE of the covariance matrix, $S$, by

 \begin{align}
 \| \hat{\Sigma} - S\|_{\infty} \leq \frac{\xi}{\alpha}. \nonumber
 \end{align}

 \end{theorem}
 \begin{proof}
 Let $\mathbf{c} = \mathbf{f} - H \hat{\mathbf{\zeta}'}$ and let $C$ be $\mathbf{c}$ reshaped such that $C \in \mathbb{R}^{p \times p}$. Then,

 \begin{align}
 \hat{L} S + S\hat{L}^T = 2I + C. \label{lyap5}
 \end{align}
  Combining (\ref{lyap5}) with $\hat{L} \hat{\Sigma} + \hat{\Sigma} \hat{L}^T = 2 I$ yields 
\begin{align}
  \hat{L}( \hat{\Sigma}-S) + (\hat{\Sigma}-S) \hat{L}^T = C,
\end{align}
  which can be written as
\begin{align}
  \text{vec}( \hat{\Sigma}-S) = (I \otimes \hat{L} + \hat{L}^T \otimes I)^{-1} \mathbf{c}. \nonumber
\end{align}
Taking the infinity norm and applying Lemma \ref{lemmasdd} gives 

 \begin{align}
\| \text{vec}(\hat{\Sigma} - S)\|_{\infty} \leq&   \|(I \otimes \hat{L} + \hat{L}^T \otimes I)^{-1} \|_{\infty} \; \| \mathbf{c} \|_{\infty} \leq \|(I \otimes \hat{L} + \hat{L}^T \otimes I)^{-1} \|_{\infty} \; \| \mathbf{c} \|_2,  
 \nonumber \\
\| \hat{\Sigma} - S\|_{\infty}  \leq & \frac{\xi}{\alpha}. \nonumber 
 \end{align} 
 \end{proof}

 \begin{corollary} \label{cor1}
 Let $\tilde{\Sigma}$ be the covariance matrix estimate obtained by solving the $l_1$-norm penalized maximum likelihood problem $ \tilde{\Sigma} = \arg \max_{Z \succ 0} -\log \det Z - \text{\emph{trace}} (SZ^{-1}) - \lambda \| Z^{-1} \|_1$. Let $\hat{\Sigma}$ be the covariance matrix estimate for which Theorem \ref{thmbound} is satisfied. Then, 
 \begin{align}
 \| \tilde{\Sigma} - \hat{\Sigma} \|_{\infty} \leq \frac{\xi}{\alpha} + \lambda. \nonumber
 \end{align} 
 \end{corollary}
\begin{proof}
From \citep{banerjee2008} we have that $\tilde{\Sigma} = S+ \tilde{U}$, where $\| \tilde{U} \|_{\infty} \leq \lambda$. Plugging $S$ in to the statement of Theorem \ref{thmbound} yields
\begin{align}
\| \hat{\Sigma} - S\|_{\infty} = \| \hat{\Sigma} - \tilde{\Sigma} + \tilde{U} \|_{\infty} \leq  \frac{\xi}{\alpha}.\nonumber
\end{align}
Recall that for arbitrary $a$, $b$ we have that $\| a \| \leq \|a+ b \| + \|b \|$. Thus,
\begin{align}
\| \hat{\Sigma} - \tilde{\Sigma}\|_{\infty} \leq \frac{\xi}{\alpha} + \lambda. \nonumber
\end{align}
 \end{proof}
 We have shown, therefore, that the maximum additive difference between an estimated covariance matrix corresponding to a GGIM and the $l_1$-norm penalized maximum likelihood estimate of covariance is bounded.

\section{GGIM estimation with sparse directed conditional dependencies }\label{sec:sparse}
It may be that a sparse interaction graph induces a dense set of directed conditional dependence relationships. In this context dense implies that most off-diagonal entries of $\Phi_{a,a}$ for all $a = \{k, l \} \in \mathcal{V}$ are nonzero. However, this violates a common assumption that conditional dependencies among real world data are sparse. 
 We call a GGIM of a Gaussian process \emph{consistent} if equation(\ref{gamma1}) is satisfied for all $a = \{k, l \} \in \mathcal{V}$.  We call a GGIM of a Gaussian process \emph{$\epsilon$-consistent} if $\| \hat{\Phi}_{a,a} - (\hat{L}_{a,a} + \Sigma_{a,b}(\Sigma_{b,b})^{-1}(  \hat{P}_{b,a} - \hat{L}_{b,a})) \| < \epsilon$ for all $a = \{k, l \} \in \mathcal{V}$. 
 

\subsection{Exact approach for minimum $l_1$ norm consistent GGIMs}\label{sec:exact}
Equation (\ref{psplit}) establishes the strong relationship between sparsity of an inverse covariance matrix, $P$, and elements $\Phi_{k,l}$, $\Phi_{l,k}$. Consequently, we approach the problem of estimating a sparse set of $2 \times 2$ matrices $\hat{\Phi}_{a,a}$ assuming that a sparse estimate $\hat{P}$ has already been determined through estimation methods that take into account the diagonal entries such those in \citet{Balmand2016}. For each unique pair of nodes, $k, l$ with $k \neq l$, assume that $\hat{P}_{k,k} \geq \hat{P}_{l,l}$ without loss of generality. Then the minimum $l_1$ norm estimate of every $\hat{\Phi}_{a,a}$ associated with $\hat{P}$ is given by setting 
\begin{align}
\hat{\Phi}_{l,k} = 0, \;\;  \hat{\Phi}_{k,l} = \hat{P}_{k,l}\Big( 1 + \frac{\hat{P}_{l,l}}{\hat{P}_{k,k}} \Big), \label{minphi}
\end{align} 
and iterating over all pairs of nodes.  As $\hat{P}_{k,k} = \frac{ \hat{\Sigma}_{l | b}}{\det(\hat{P}_{a,a})}$ with $a = \{k,l \} $, we can interpret equation (\ref{minphi}) as assuming that the nodes with pairwise higher conditional variance are directed conditionally dependent on the nodes with lower conditional variance.

The GGIM associated with $\hat{\Phi}$ can then be backed out analytically be solving a system of linear equations. Assigning values to $\hat{\Phi}_{k,l}$ and $\hat{\Phi}_{l,k}$ corresponds to defining the single entry $L_{k,l}$ (or $L_{l,k}$). In other words, $\hat{\Phi}$ establishes $\frac{p}{2}(p-1)$ constraint equations of the form
\begin{align}
\hat{L}_{k,l} = \hat{\Phi}_{k,l} - \hat{\Sigma}_{k,b}(\hat{\Sigma}_{b,b})^{-1}(\hat{P}_{b,l} - \hat{L}_{b,l}). \label{leqs}
\end{align}

 Recall that the GGIM estimate satisfies $\hat{L} = (I + \kappa)\hat{P}$, with $\kappa$ skew-symmetric. The $(k,l)$ entry is then $\hat{L}_{k,l} = \hat{P}_{k,l} + \kappa_{k,l} \hat{P}_{l,l} + \kappa_{k,b}\hat{P}_{b,l}$. Thus, the problem of finding the GGIM estimate associated with a sparse directed conditional dependencies becomes determining the $\frac{p}{2}(p - 1)$ independent values of $\kappa$ such that the equations for $\frac{p}{2}(p-1)$ values of $\hat{L}$ are satisfied. For each pair $(k,l)$ and $b = \{ \mathcal{V}\backslash \{k,l\} \}$ we have the following equation
 \begin{align}
 \hat{P}_{k,l}& = \hat{\Phi}_{k,l} + \hat{\Sigma}_{k,b}(\hat{\Sigma}_{b,b})^{-1}(\kappa_{b,a} \hat{P}_{a,l} + \kappa_{b,b} \hat{P}_{b,l}) - \kappa_{k,l} \hat{P}_{l,l} - \kappa_{k,b}\hat{P}_{b,l}. \label{cij}
 \end{align}
Equation (\ref{cij}) can be equivalently written in the form $W\mathbf{y} = \mathbf{z}$ where $\mathbf{y} \in \mathbb{R}^{\frac{p}{2}(p-1)}$ is the vector that stacks all $\kappa_{k,l}$, $\mathbf{z}\in \mathbb{R}^{\frac{p}{2}(p-1)}$ is the vector that stacks all $\hat{P}_{k,l}$, and entries in $W$ are corresponding coefficients on entries of  $\kappa$ from equation (\ref{cij}). Matrix $W$ is invertible and the independent entries of $\kappa$ can be found by solving $\mathbf{y} = W^{-1}\mathbf{z}$. The corresponding GGIM estimate is the graph induced by the matrix $\hat{L} = (I + \kappa)\hat{P}$. 

In general, the resulting GGIM estimate will not be sparse. However, if desired we can manually trade off sparsity of $\hat{\Phi}$ and sparsity of $\hat{L}$ to fit the application at hand. For every unique pair of nodes $(k, l)$ with $k \neq l$ we can explicitly set one of $\{ \hat{\Phi}_{k,l}, \; \hat{\Phi}_{l,k}, \; \hat{L}_{k,l},\; \hat{L}_{l,k} \}$ equal to zero and solve for the elements of $\kappa$. Thereby, we can incorporate prior knowledge and still guarantee that the GGIM and directed conditional independence relationships will be consistent. 

\subsection{Approximate approach for minimum $l_1$ norm  $\epsilon$-consistent GGIMs}
It may be that the manual trade of of sparsity between directed conditional dependence and GGIM described in Section \ref{sec:exact} will not result in a model that is satisfactorily sparse. This is because very exact and complex relationships between edge weights have to be satisfied in order for both models to be consistent and sparse. Even if these relationships are satisfied in the true underlying model, errors in the sample covariance matrix and estimated precision matrix may prevent us from recovering them. Therefore, we would like to consider $\epsilon$-consistent GGIMs as a way to increase sparsity and account for errors. In what follows, we briefly mention one approach for estimation.

We assume that a sparse estimate $\hat{P}$ has been found, and all matrices $\hat{\Phi}_{a,a}$ are given by equation (\ref{minphi}) for all $\{k,l\} \in \mathcal{V}$. Then, a $\epsilon$-consistent GGIM can be found be rewriting equation (\ref{leqs}) as a matrix equation in all off-diagonal elements of $L$ and solving the corresponding LASSO problem
\begin{align*}
\mathbf{\hat{v}} = \min_{\mathbf{v}} \| \mathbf{u} - T \mathbf{v} \|_2^2 + \lambda \|\mathbf{v} \|_1,
\end{align*} 
where $\mathbf{v}$ is the vector of all off-diagonal elements of $\hat{L}$, $T$ is a matrix of appropriate coefficients from (\ref{leqs}), $\mathbf{u}$ is the vector of all terms that are not dependent on any $\hat{L}_{i,j}$, and $\lambda \geq 0$. Tuning $\lambda$ such that the restructured $\hat{L}$ satisfies $\| \hat{\Phi}_{a,a} - (\hat{L}_{a,a} + \Sigma_{a,b}(\Sigma_{b,b})^{-1}(  \hat{P}_{b,a} - \hat{L}_{b,a})) \| < \epsilon$ for all $a = \{k, l \} \in \mathcal{V}$ results in an  $\epsilon$-consistent GGIM.

\section{Final Remarks}\label{sec:final}
In this paper, we have defined a new graphical model from Gaussian data, the Gaussian graphical interaction model. Stationary Gaussian processes on graphs provides a theoretical basis for the model. By leveraging the Lyapunov equations between steady-state covariance matrices and the Laplacian matrices of interaction graphs, we are able to characterize families of directed graphs with equivalent steady-state behavior. Moreover, we have shown that our approach leads to an intuitive definition of directed conditional independence. We have shown that the problem of learning GGIMs can be formulated as a LASSO problem when the measure of sparsity is taken to be the $l_1$-norm. Moreover, we have proven that the estimated covariance matrix associated with a learned GGIM is bounded with respect to the standard $l_1$-norm penalized maximum log-likelihood estimate. The GGIM opens up new avenues for the modeling of directed relationships from Gaussian data and significantly expand upon the capabilities of Gaussian graphical modeling. 

\acks{This work was supported by the Alexander von Humboldt Foundation with funds from the German Federal Ministry of Education and Research (BMBF).}

\begin{appendices}
\section{GGIMs with positive semi-definite covariance matrices} \label{app1}
In this section we discuss GGIMs when the sample covariance matrix is positive semi-definite. In this case we borrow from \citep{Fitch2019} and apply simple projections to accommodate the rank deficiency of the sample covariance matrix. We begin by reviewing some definitions from \citep{Fitch2019}.

Let $\mathbf{e}_p^{(k)}$ be the $k$th standard basis vector for $\mathbb{R}^p$. Let $P_p = I_p - \frac{1}{p}\mathbf{1}_p\mathbf{1}_p^T$. Let $\mathbf{1}^{\perp}_p =$span$\{\mathbf{1}_p\}^{\perp}$ be the subspace of $\mathbb{R}^p$ perpendicular to $\mathbf{1}_p$. Let $Q \in \mathbb{R}^{(p - 1) \times p}$ be a matrix with rows that form an orthonormal basis for $\mathbf{1}^{\perp}_p $. Then the following properties hold,
\begin{align}
Q\mathbf{1}_p = \mathbf{0}, \;\; QQ^T = I_{p-1}, \;\; Q^TQ=P_p. \label{Q}
\end{align}

The \emph{reduced Laplacian} matrix is defined as $\bar{L} = QLQ^T$, and characterizes the Laplacian matrix on $\mathbf{1}^{\perp}_p $. $\bar{L}$ has the same eigenvalues as $L$ except for the 0 eigenvalue and is therefore invertible if the graph is connected \citep{young2010}. We say that a graph is connected if there exists at least one globally reachable node, $k$. In other words, there is a directed path from every node $i$ to $k$. 

Let $\Psi \in \mathbb{R}^{p \times p}$ be a projection matrix onto $\mathbf{1}^{\perp}_p $, where $\Psi$ is not necessarily an orthogonal projection matrix and $\Psi L = L$, that is the image of $L$ is contained in the kernel of $(\Psi-I_p)$. The intuition behind the matrix $\Psi$ is that it characterizes the set of globally reachable nodes. 

Due to the rank deficiency of $\Sigma$, we consider the Layapunov equation  (\ref{lyap1}) restricted to $\mathbf{1}^{\perp}_p $. That is, we are looking for the family of reduced Laplacian matrices, $\bar{L}$ for which the solution, $\bar{\Sigma}$ to $\bar{L} \bar{\Sigma} + \bar{\Sigma} \bar{L} ^T = 2 I_{p-1}$, is $\bar{\Sigma} = Q \Sigma Q^T$. The following proposition follows from statements in \citep{Fitch2019}.
\begin{proposition} \label{prop2}
Consider a stationary Gaussian process $\mathbf{x}  \in \mathbb{R}^p$ with multivariate Gaussian distribution $\mathcal{N}(0,\Sigma)$ with $\Sigma$ positive semi-definite. Then the Laplacian matrices corresponding to the family of steady-state equivalent (on $\mathbf{1}^{\perp}_p $)  interaction graphs between the $p$ variables of $\mathbf{x}$ are given as a function of $ \Psi \in \mathbb{R}^{p \times p}$ projection matrices and $\kappa \in \mathbb{R}^{p \times p}$ skew-symmetric matrices, by
\begin{align}
L(\Psi, \kappa) = \Psi (I_p + \kappa)\Sigma^{+}, \label{lkp}
\end{align}
where $\Sigma^{+}$ denotes the Moore-Penrose pseudo-inverse of $\Sigma$. 
\end{proposition}

For the problem of learning GGIMs given a positive semi-definite $\Sigma$, we follow a nearly identical approach to Section \ref{sec:ggiml} and focus on estimating the Laplacian matrix directly. The associated Lyapunov equation is 
\begin{align}
Q L Q^T \bar{\Sigma} + \bar{\Sigma} Q L Q^T = 2I_{p-1},
\end{align}
which is linear in $L$ and can easily be written as a LASSO problem similar to (\ref{obj2}). The estimates $\hat{\Psi}$ and $\hat{\kappa}$ can be uniquely determined from a directed Laplacian estimate $\hat{L}$. 
\end{appendices}

\bibliography{GGM_revised}

\begin{thebibliography}{29}
\providecommand{\natexlab}[1]{#1}
\providecommand{\url}[1]{\texttt{#1}}
\expandafter\ifx\csname urlstyle\endcsname\relax
  \providecommand{\doi}[1]{doi: #1}\else
  \providecommand{\doi}{doi: \begingroup \urlstyle{rm}\Url}\fi

\bibitem[Anderson(2003)]{Anderson2003}
T.~W. Anderson.
\newblock \emph{An introduction to Multivarite statistical analysis}.
\newblock John Wiley and Sons, 3 edition, 2003.

\bibitem[Arnold(1992)]{arnold1992}
L.~Arnold.
\newblock \emph{Stochastic Differential Equations: Theory and Applications}.
\newblock Krieger Publishing Complany, 1992.

\bibitem[Balmand and Dalalyan(2016)]{Balmand2016}
S.~Balmand and A.~S. Dalalyan.
\newblock On estimation of the diagonal elements of a sparse precision matrix.
\newblock \emph{Electronic Journal of Statistics}, 10\penalty0 (1):\penalty0
  1551--1579, 2016.

\bibitem[Banerjee et~al.(2008)Banerjee, Ghaoui, and d'Aspremont]{banerjee2008}
O.~Banerjee, L.~E. Ghaoui, and A.~d'Aspremont.
\newblock Model selection through sparse maximum likelihood estimation for
  multivariate gaussian or binary data.
\newblock \emph{Journal of Machine Learning Research}, 9:\penalty0 485--516,
  2008.

\bibitem[Barber(2012)]{barber2012}
D.~Barber.
\newblock \emph{Bayesian Reasoning and Machine Learning}.
\newblock Cambridge University Press, 2012.

\bibitem[Barnett and Storey(1967)]{Barnett1967}
S.~Barnett and C.~Storey.
\newblock Analysis and synthesis of stability matrices.
\newblock \emph{J. Diff. Eq.}, 3:\penalty0 414--422, 1967.

\bibitem[Bollen(1989)]{Bollen1989}
K.~Bollen.
\newblock \emph{Structure Equations with Latent Variables}.
\newblock Wiley Publishing, 1989.

\bibitem[Chickering(2002)]{chickering2002}
D.~M. Chickering.
\newblock Learning equivalence classes of bayesian-network structures.
\newblock \emph{Journal of Machine Learning Research}, 2:\penalty0 445--498,
  2002.

\bibitem[Fitch(2019)]{Fitch2019}
K.~Fitch.
\newblock Effective resistance preserving directed graph symmetrization.
\newblock \emph{SIAM Journal on Matrix Analysis and Applications}, 40\penalty0
  (1):\penalty0 49--65, 2019.

\bibitem[Friedman et~al.(2008)Friedman, Hastie, and Tibshirani]{Friedman2008}
J.~Friedman, T.~Hastie, and R.~Tibshirani.
\newblock Sparse inverse covariance estimation with the graphical lasso.
\newblock \emph{Biostatistics}, 9\penalty0 (3):\penalty0 432--441, 2008.

\bibitem[G.Varando and Hansen(2020)]{Varando2020}
G.Varando and N.~R. Hansen.
\newblock Graphical continuous lyapunov models.
\newblock arXiv:2005.10483, 2020.

\bibitem[Hassan-Moghaddam et~al.(2016)Hassan-Moghaddam, Dhingra, and
  Jovanovic]{Hassan2016}
S.~Hassan-Moghaddam, N.~K. Dhingra, and M.~R. Jovanovic.
\newblock Topology identification of undirected consensus networks via sparse
  inverse covariance estimation.
\newblock In \emph{IEEE 55th Conference on Decision and Control}, pages
  4624--4629, 2016.

\bibitem[Hastie et~al.(2015)Hastie, Tibshirani, and Wainwright]{Hastie2015}
T.~Hastie, R.~Tibshirani, and M.~Wainwright.
\newblock \emph{Statistical Learning with Sparsity: The LASSO and
  Generalizations}.
\newblock Chapman and Hall/CRC, 2015.

\bibitem[Hsieh et~al.(2011)Hsieh, Dhillon, Ravikumar, and Sustik]{Hsieh2011}
C.~J. Hsieh, I.~S. Dhillon, P.~K. Ravikumar, and M.~A. Sustik.
\newblock Sparse inverse covariance matrix estimation using quadratic
  approximation.
\newblock In \emph{Advances in neural information processing systems 24}, pages
  2330--2338, 2011.

\bibitem[Kruschwitz et~al.(2015)Kruschwitz, List, Waller, Rubinov, and
  Walter]{Kruschwitz2015}
J.~D. Kruschwitz, D.~List, L.~Waller, M.~Rubinov, and H.~Walter.
\newblock Graphvar: a user-friendly toolbox for comprehensive graph analyses of
  functional brain connectivity.
\newblock \emph{Journal of neuroscience methods}, 245:\penalty0 107--115, 2015.

\bibitem[Lauritzen and Spiegelhalter(1988)]{lauritzen1988}
S.~L. Lauritzen and D.~J. Spiegelhalter.
\newblock Local computations with probabilities on graphical structures and
  their appications to expert systems.
\newblock \emph{Journal of teh Royal Statistical Society B}, 50\penalty0
  (2):\penalty0 157--224, 1988.

\bibitem[Lauritzen(1996)]{Lauritzen1996}
Steffen~L. Lauritzen.
\newblock \emph{Graphical Models}.
\newblock Clarendon Press, 1996.

\bibitem[Marjanovic and Hero(2015)]{Marjanovic2015}
G.~Marjanovic and A.~O. Hero.
\newblock $l_0$ sparse inverse covariance estimation.
\newblock \emph{IEEE Transactions on Signal Processing}, 63\penalty0
  (12):\penalty0 3218--3231, 2015.

\bibitem[Marjanovic and Solo(2014)]{Marjanovic2014}
G.~Marjanovic and V.~Solo.
\newblock On $l_q$ optimization and sparse inverse covariance selection.
\newblock \emph{IEEE Transactions on Signal Processing}, 62\penalty0
  (7):\penalty0 1644--1654, 2014.

\bibitem[Oztoprak et~al.(2012)Oztoprak, Nocedal, Rennie, and
  Olsen]{Oztoprak2012}
Figen Oztoprak, Jorge Nocedal, Steven Rennie, and Peder~A Olsen.
\newblock Newton-like methods for sparse inverse covariance estimation.
\newblock In \emph{Advances in Neural Information Processing Systems 25}, pages
  755--763, 2012.

\bibitem[Sachs et~al.(2005)Sachs, Perez, Pe'er, Lauffenburger, and
  Nolan]{sachs2005}
K.~Sachs, O~Perez, D~Pe'er, D.~A. Lauffenburger, and G.~P. Nolan.
\newblock Causal protein-signaling networks derived from multiparameter
  single-cell data.
\newblock \emph{Science}, 308\penalty0 (5721):\penalty0 523--529, 2005.

\bibitem[Souly and Shah(2016)]{Souly2016}
N.~Souly and M.~Shah.
\newblock Scene labeling using sparse precision matrix.
\newblock In \emph{IEEE Conference on Computer Vision and Pattern Recognition},
  pages 3650--3658, 2016.

\bibitem[Varga(1976)]{varga1976}
R.~S. Varga.
\newblock On diagonal dominance arguments for bounding $\|{A}^{-1}
  \|_{\infty}$.
\newblock \emph{Linear Algebra and its Applications}, 14:\penalty0 211--217,
  1976.

\bibitem[Whittaker(2009)]{whittaker2009}
J.~Whittaker.
\newblock \emph{Graphical Models in Applied Multivariate Statistics}.
\newblock Wiley Publishing, 2009.

\bibitem[Wu and Lange(2008)]{Wu2008}
T.~T. Wu and K.~Lange.
\newblock Coordinate descent algorithms for lasso penalized regression.
\newblock \emph{Annals of Applied Statistics}, 2\penalty0 (224-244), 2008.

\bibitem[Young et~al.(2010)Young, Scardovi, and Leonard]{young2010}
G.~F. Young, L.~Scardovi, and N.~E. Leonard.
\newblock Robustness of noisy consensus dynamics with directed communication.
\newblock In \emph{In proc. ACC}, pages 6312--6317, 2010.

\bibitem[Young et~al.(2016)Young, Scardovi, and Leonard]{young20161}
G.~F. Young, L.~Scardovi, and N.~E. Leonard.
\newblock A new notion of effective resistance for directed graphs---part {I}:
  Definition and properties.
\newblock \emph{IEEE Transactions on Automatic Control}, 61\penalty0
  (7):\penalty0 1727 -- 1736, 2016.

\bibitem[Zhang and Spirtes(2003)]{Zhang2003}
J.~Zhang and P.~Spirtes.
\newblock Strong faithfulness and uniform consistency in causal inference.
\newblock In \emph{Proceedings of Uncertaintly in Artificial Intelligence},
  pages 632--639, 2003.

\bibitem[Zhang and Fung(2013)]{zhang2013}
W.~Zhang and P.~Fung.
\newblock Sparse inverse covariance matrices for low resource speech
  recognition.
\newblock In \emph{IEEE Transactions on Audio, Speech, and Language
  Processing}, pages 659--668, 2013.

\end{thebibliography}

\end{document}